\documentclass[draftcls,onecolumn,11pt]{IEEEtran}

\usepackage{amsmath,amstext,amsthm,amssymb,epsf}
\usepackage{fullpage,latexsym}
\usepackage[pdftex]{graphicx}
\usepackage{setspace}
\numberwithin{equation}{section} 

 \newtheorem{lemma}{Lemma}[section]
 \newtheorem{theorem}[lemma]{Theorem}

 \newtheorem{claim}[lemma]{Claim}

 \newtheorem{definition}[lemma]{Definition}
 \newtheorem{rem}[lemma]{Remark}
\newenvironment{remark}{\begin{rem}}{\hspace*{\fill}$\diamondsuit$\end{rem}}
 \newtheorem{ex}[lemma]{Example}

\renewcommand{\emptyset}{\varnothing}

\begin{document}

\title{Normalized Compression Distance of Multisets with Applications}
\author{Andrew R. Cohen\thanks
{Andrew Cohen is with the Department of Electrical and Computer Engineering, 
Drexel University.
Address: A.R. Cohen, 3120--40 Market Street, 
Suite 313, Philadelphia, PA 19104,
USA. Email: {\tt acohen@coe.drexel.edu}} 
and Paul M.B. Vit\'{a}nyi
\thanks{
Paul Vit\'{a}nyi is with the national research center for mathematics 
and computer science in the Netherlands (CWI),
and the University of Amsterdam.
Address:
CWI, Science Park 123,
1098XG Amsterdam, The Netherlands.
Email: {\tt Paul.Vitanyi@cwi.nl}
}}


\maketitle

\begin{abstract}
Normalized compression distance (NCD) is a parameter-free, feature-free,
alignment-free, similarity measure between a pair of finite objects
based on compression. However, it
is not sufficient for all applications.
We propose an NCD of 
finite nonempty multisets (a.k.a. multiples) of 
finite objects that is also a metric. 
Previously, attempts to obtain such an NCD failed.
We cover the entire trajectory from theoretical underpinning to
feasible practice.
The new NCD for multisets is applied to retinal progenitor cell 
classification questions and to related synthetically generated data that were
earlier treated with the pairwise NCD. With the new
method we achieved significantly better results. Similarly for
questions about axonal organelle transport.
We also applied the new NCD to handwritten digit recognition and 
improved classification accuracy significantly over that of pairwise NCD by incorporating 
both the pairwise and NCD for multisets.
In the analysis we use the incomputable
Kolmogorov complexity that
for practical purposes is 
approximated from above by the length
of the compressed version of the file involved, using
a real-world compression program.

{\em Index Terms}---
Normalized compression distance, multisets or multiples, pattern recognition, 
data mining, similarity, classification, Kolmogorov complexity, 
retinal progenitor cells, synthetic data, organelle transport, handwritten character recognition
\end{abstract}

\section{Introduction}
\label{sect.intro}

To define the information in a {\em single} finite object one commonly uses
the Kolmogorov complexity \cite{Ko65} of that object (finiteness is taken
as understood in the sequel).
Information distance \cite{BGLVZ98} is the information 
required to transform one in the other, or vice versa, among a 
{\em pair} of objects. For research in the theoretical direction
see among others \cite{Mu02}. Here we are more concerned with
normalizing it to obtain the so-called similarity metric 
and subsequently approximating the Kolmogorov complexity through real-world
compressors \cite{Li03}. This leads to the normalized compression distance
(NCD) which is theoretically analyzed and applied to general hierarchical
clustering in \cite{CV04}. The NCD is
parameter-free, feature-free, and alignment-free,
and has found many applications in pattern recognition, phylogeny, clustering,
and classification, for example 
\cite{AS05,KJ04,KKKP06,NPAetal08,NPLetal08,Co09,Co10,WLB08} and the many
references in Google Scholar to \cite{Li03,CV04}.
Another application is to 
objects that are only represented by name, 
or objects that are abstract like `red,' `Einstein,' 
`three.' In this case the distance uses background information
provided by Google or any search engine that produces aggregate page 
counts. It discovers the `meaning' of words and phrases in the sense
of producing a relative semantics \cite{CV07}.  
The question arises of the
shared information between many objects
instead of just a pair of objects.

\subsection{Related Work}\label{sect.relwork}

In \cite{Li08} the notion is introduced of the information required to
go from any object in a multiset of objects to any other object in the
multiset. This
is applied to extracting the essence from, for example, a finite 
nonempty multiset
of internet news items,  
reviews of electronic cameras, tv's, and so on, in a way
that works better than other methods. 
Let $X$ denote a finite nonempty multiset of 
$m$ finite binary strings defined by (abusing the set notation)
$X=\{x_1, \ldots, x_m\}$, the constituting
elements (not necessarily all different) ordered length-increasing lexicographic.
We use multisets and not sets, since if $X$ is a set then all of its
members are different while we are interested in the situation were
some or all of the objects are equal. 
Let $U$ be the reference universal Turing machine, for convenience the prefix
one as in Section~\ref{sect.prel}.
We define the {\em information distance} in $X$
by $E_{\max} (X) = \min \{|p|: U(x_i,p,j)=x_j$ for all $x_i,x_j \in X$\}.
It is shown in \cite{Li08}, Theorem 2, that 
\begin{equation}\label{eq.li08}
E_{\max}(X)=
\max_{x:x \in X} K(X|x),
\end{equation}
up to an additive term $O(\log K(X))$.
Define $E_{\min}(X) = \min_{x:x \in X} K(X|x)$.
Theorem 3 in \cite{Li08} states that 
\begin{equation}\label{eq.li083}
E_{\min}(X) \leq E_{\max}(X) \leq
\min_{i: 1 \leq i \leq m} 
\sum_{x_i,x_k \in X \; \& \; k \neq i} E_{\max} (x_i,x_k),  
\end{equation}
up to a logarithmic additive term. 
The information distance in \cite{BGLVZ98}
between strings $x_1$ and $x_2$ is denoted by $E_1(x_1,x_2) 
= \max\{K(x_1|x_2),K(x_2|x_1)\}$.
Here we use the notation $\max_{x:x \in X} K(X|x)$.
The two coincide for $|X|=2$ since $K(x,y|x)=K(y|x)$ up
to an additive constant term.
In \cite{Vi11} this notation was introduced and
the following results were obtained for finite nonempty multisets.
The maximal overlap of information, concerning the remarkable property
that the information needed to go from
any member $x_j$ to any other member $x_k$ in a multiset $X$ 
can be divided in two parts: a single string
of length $\min_i K(X|x_i)$ and a special string of length $\max_i (K(X|x_i)
-\min_i K(X|x_i)$ possibly depending on $j$ and some 
logarithmic additive terms possibly depending on $j,k$.
Furthermore, 
the minimal overlap property,
the metricity property, the 
universality property, and 
the not-subadditivity property.
With respect to normalization of the information distance of
multisets only abortive attempts were given.
A review of some of the above is \cite{Li11}.

\subsection{Results}
For many applications we require a normalized and computable version
of the information distance for finite nonempty multisets of finite objects.
For instance,
classifying an object into one or another of disjoint classes we aim
for the class of which the NCD for multisets grows the least.
We give preliminaries in Section~\ref{sect.prel}.
The normalization of the information distance for multisets which
did not succeed in \cite{Vi11} is analyzed and performed in 
Section~\ref{sect.nid}. In particular it is proved to be a metric.
We sometimes require metricity since otherwise the results may be inconsistent
and absurd.
Subsequently we proceed to the practically feasible compression distance
for multisets and prove this
to be a metric, Section~\ref{sect.cd}. Next, this compression distance
is normalized and proved to retain the metricity, Section~\ref{sect.ncd}.
We go into the question of how to compute this in Section~\ref{sect.comp},
how to apply this to classification in Section~\ref{sect.ancd}, 
and then treat the applications Section~\ref{sect.app}.
We applied the new NCD for multisets to retinal progenitor 
cell classification questions, Section~\ref{sect.rpc},
and to synthetically generated data, Section~\ref{sect.synth}, that were
earlier treated with the pairwise NCD. Here we get significantly better
results. This was also the case for
questions about axonal organelle transport, Section~\ref{sect.aot}.
We also applied the NCD for multisets to classification of 
handwritten digits, Section~\ref{sect.NIST}. 
Although the NCD for multisets did not improve on the accuracy 
of the pairwise NCD for this application, 
classification accuracy was much improved over either method individually by
combining the pairwise and multisets NCD. We treat the data, software, 
and machines used for the applications in Section~\ref{sect.means}. 
We finish with conclusions in 
Section~\ref{sect.concl}.

\section{Preliminaries}\label{sect.prel}
We write {\em string} to mean a finite binary string,
and $\epsilon$ denotes the empty string.
The {\em length} of a string $x$ (the number of bits in it)
is denoted by $|x|$. Thus,
$|\epsilon| = 0$.
We identify strings with natural numbers
by associating each string with its index
in the length-increasing lexicographic ordering according to the scheme
$
( \epsilon , 0),  (0,1),  (1,2), (00,3), (01,4), (10,5), (11,6), 
\ldots . 
$
In this way the Kolmogorov complexity in the next section can be about finite binary 
strings or natural numbers.
\subsection{Kolmogorov Complexity} 
The Kolmogorov complexity is the information in a single finite object
\cite{Ko65}.
Informally, the Kolmogorov complexity of a finite binary string
is the length of the shortest string from which the original
can be lossless reconstructed by an effective
general-purpose computer such as a particular universal Turing machine.
Hence it constitutes a lower bound on how far a
lossless compression program can compress.
For technical reasons we choose Turing machines with a separate 
read-only input tape that is scanned from left to right without backing up, 
a separate work tape on which the computation takes place, 
and a separate output tape. All tapes are divided into squares
and are semi-infinite. Initially,
the input tape contains a semi-infinite binary string with one bit per square
starting at the leftmost square, and all heads scan the leftmost square 
on their tapes. Upon halting, the initial segment $p$ 
of the input that has been scanned is called the input ``program'' 
and the contents of the output tape is called the ``output.'' 
By construction, the set of halting programs is prefix free. 
An standard enumeration of such Turing machines $T_1,T_2, \ldots$ 
contains a universal machine $U$ such that $U(i,p)=T_i(p)$ for all
indexes $i$ and programs $p$. (Such universal machines are called 
``optimal'' in contrast with universal machines like $U'$ with $U'(i,pp)=T_i(p)$
for all $i$ and $p$, and $U'(i,q)=1$ for $q \neq pp$ for some $p$.) 
We call $U$ the {\em reference universal prefix Turing machine}.
This leads to the definition of ``prefix Kolmogorov complexity''
which we shall designate simply as ``Kolmogorov complexity.''

Formally, the {\em conditional Kolmogorov complexity}
$K(x|y)$ is the length of the shortest input $z$
such that the reference universal prefix Turing machine $U$ on input $z$ with
auxiliary information $y$ outputs $x$. The
{\em unconditional Kolmogorov complexity} $K(x)$ is defined by
$K(x|\epsilon)$ where $\epsilon$ is the empty string.
In these definitions both $x$ and $y$ can consist of strings into which 
finite multisets of finite binary strings are encoded.  

Theory and applications are given in the textbook \cite{LV08}.
Here we give some relations that are needed in the paper.
The {\em information about $x$ contained in $y$} is defined as
$I(y:x)=K(x)-K(x | y)$. A deep, and very useful, result
holding for both plain complexity and
prefix complexity, due to L.A. Levin and A.N. Kolmogorov \cite{ZL70} called
{\em symmetry of information} states that
\begin{equation}\label{eq.soi}
K(x,y)=K(x)+K(y | x) = K(y)+K(x | y),
\end{equation}
with the equalities holding up to a $O(\log K)$  additive term.
Here, $K=\max\{K(x),K(y)\}$.
Hence, up to an additive logarithmic  term $I(x:y) = I(y:x)$ and we
call this the {\em mutual (algorithmic) information} between $x$ and $y$.

\subsection{Multiset} 
A multiset is also known as {\em bag}, {\em list}, or {\em multiple}. 
A {\em multiset} is a generalization of the notion of set. 
The members are allowed to appear more than once. 
For example, if $x\neq y$ then $\{x,y\}$ is a set, but $\{x,x,y\}$
and $\{x,x,x,y,y\}$ are multisets, with abuse of the set notation. 
We also abuse the 
set-membership notation by using it 
for multisets by writing $x \in \{x,x,y\}$ and $z \not\in \{x,x,y\}$
for $z \neq x,y$. Further, $\{x,x,y\} \setminus \{x\} = \{x,y\}$. If $X,Y,Z$
are multisets and $Z$ is nonempty and $X=YZ$, then
we write $Y \subset X$.  
For us, a multiset is finite and nonempty such as 
$\{x_1 , \ldots ,x_n\}$ with $0 < n < \infty$ and the members are 
finite binary strings in
length-increasing lexicographic order. If $X$ is a multiset, then some
or all of its elements may be equal. $x_i \in X$ means 
that ``$x_i$ is an element
of multiset $X$.'' 
With $\{x_1 ,\ldots , x_m\} \setminus \{x\}$ we mean the multiset
$\{x_1 \ldots x_m\}$ 
with one occurrence of $x$ removed.

The finite binary strings, finiteness,
and length-increasing lexicographic order allows us to assign
a unique Kolmogorov complexity to a multiset.
The conditional prefix Kolmogorov complexity $K(X|x)$ of a multiset
$X$ given an element $x$ is the length of a shortest program $p$
for the reference universal Turing machine that with input $x$
outputs the multiset $X$. The prefix Kolmogorov complexity $K(X)$ of a multiset
$X$ is defined by $K(X| \epsilon )$. 
One can also put multisets in the conditional such as $K(x|X)$ or 
$K(X|Y)$.
We will use the straightforward laws $K(\cdot|X,x)=K(\cdot|X)$
and $K(X|x)=K(X'|x)$ up to an additive constant term, for $x \in X$
and $X'$ equals the multiset $X$ with one occurrence of the element $x$ deleted.

\subsection{Information Distance}
The information distance in a multiset $X$ ($|X| \geq 2$) is given by
\eqref{eq.li08}.
To obtain the {\em pairwise information distance} in \cite{BGLVZ98} 
we take $X=\{x_1,x_2\}$ in \eqref{eq.li08}. The resulting formula is equivalent
to $E_{\max}(x_1,x_2) = \max\{K(x_1|x_2), K(x_2|x_1) \}$ up to a logarithmic
additive term.

\subsection{Metricity}\label{metric}
Let ${\cal X}$
be the set of length-increasing lexicographic ordered nonempty finite
multisets of finite binary strings. 
A {\em distance function} $d$
on ${\cal X}$ is defined by $d:{\cal X} \rightarrow {\cal R}^+$ where 
${\cal R}^+$ is the set of nonnegative real numbers. 
Define $Z=XY$ if
$Z$ is the multiset consisting of the elements of the multisets $X$ and $Y$ and
the elements of $Z$ are ordered length-increasing lexicographic.
A distance function
$d$ is a {\em metric} if 
\begin{enumerate}
\item
{\em Positive definiteness}: $d(X)=0$ if all elements of $X$ are equal
and $d(X) > 0$ otherwise.
\item
{\em Symmetry}: $d(X)$ is invariant
under all permutations of $X$.
\item
{\em Triangle inequality}: $d(XY) \leq d(XZ)+d(ZY)$.
\end{enumerate}
We recall Theorem 4.1 and Claim 4.2 from  \cite{Vi11}.
\begin{theorem}\label{theo.metric}
The information distance for multisets $E_{\max}$ is a metric 
where the (in)equalities hold up to a $O(\log K)$ additive term. 
Here $K$ is the largest quantity
involved in each metric (in)equality 1) to 3), respectively.
\end{theorem}
\begin{claim}\label{claim.metric}
\rm
Let $X,Y,Z$ be three nonempty finite multisets of finite binary strings
and $K=K(X)+K(Y)+K(Z)$. Then,
$E_{\max} (XY) \leq E_{\max} (XZ) + E_{\max} (ZY)$ up to an $O(\log K)$
additive term.
\end{claim}

\section{Normalized Information Distance}\label{sect.nid}

The quantitative difference in a certain feature between many objects
can be considered as an {\em admissible} distance, provided it is 
upper semicomputable and satisfies a density condition 
for every $x \in \{0,1\}^*$ (to exclude distances like $D(X)=1/2$
for every multiset $X$):
\begin{equation}\label{eq.density}
 \sum_{X: x \in X \; \& \; D(X) > 0} 2^{-D(X)} \leq  1. 
\end{equation}
Thus, for the density
condition on $D$ we consider only multisets $X$ with $|X| \geq 2$ and not
all elements of $X$ are equal. Moreover,
we consider only distances that are upper semicomputable,
that is,  they are computable in some broad sense
(they can be computably approximated from above).
Theorem 5.2 in \cite{Vi11} shows that
$E_{\max}$ (the proof shows this actually for $\max_{x\in X} \{K(X|x)\}$)
is universal in
that among all admissible multiset distances it is always least up
to an additive constant.
That is, it accounts for the dominant feature in which the elements of 
the given multiset are alike. 

Admissible distances as defined above
are absolute, but if we want to express similarity,
then we are more interested in relative ones.
For example, if a multiset $X$ of strings of each about $1{,}000{,}000$ 
bits have pairwise information distance $1{,}000$ bits to each other,
then we are inclined to think that those strings are relatively
similar. But if a multiset $Y$ consists of strings of each
about $1{,}200$ bits and each two strings in it have a pairwise
information distance of $1{,}000$ bits,
then we think the  strings in $Y$ are very different. 
In the first case $E_{\max}(X) \approx 1{,}000|X|+O(1)$,
and in the second case $E_{\max}(Y) \approx 1{,}000|Y|+O(1)$. 
Possibly $|X| \approx |Y|$.
The information distances in the multisets
are about the same.

To express similarity we need 
to normalize the universal information distance
$E_{\max}$.
It should give a similarity
with distance 0 when the objects in a multiset are maximally similar 
(that is, they are equal)
and distance 1 when
they are maximally dissimilar.
Naturally, we desire the normalized version of the 
universal multiset information
distance metric to be also a metric. 

For pairs of objects $x,y$
the normalized version $e$ of $E_{\max}$ defined by
\begin{equation}\label{eq.pairs}
e(x,y) = \frac{\max\{ K(x,y|x),K(x,y|y\}}{\max \{K(x),K(y)\}}
\end{equation}
takes values in $[0,1]$ up to an additive term
of $O(1/K(x,y))$. It is a metric up to additive terms
$O((\log K)/K)$ with $K$ denotes the maximum of the Kolmogorov complexities
involved in each of the metric (in)equalities, respectively.
A normalization formula for multisets
of more than two elements ought to reduce to that of \eqref{eq.pairs}
for the case of multisets of two elements. 
\begin{remark}\label{rem.ex}
\rm
For example  set $A=\{x\}, B=\{y,y\}, C=\{y\}, K(x) = n, K(x|y)=n, 
K(y)=0.9n$ and by using \eqref{eq.soi} we have $K(x,y) = 1.9n, K(y|x)=0.9n$. 
The most natural definition is a generalization of \eqref{eq.pairs}: 
\begin{equation}\label{eq.defp}
e_1(X) = \frac{\max_{x\in X}\{K(X|x)\}}{\max_{x \in X}\{K(X \setminus \{x\})\}}.
\end{equation}
But we find
$e_1(AB)=K(x|y)/K(x,y)=n/1.9n \approx 1/2$, and
$e_1(AC)=K(x|y)/K(x)=n/n=1$,
$e_1(CB)= K(y|y)/K(y)=0/0.9n = 0$, and
the triangle inequality is violated. 
But with $A=\{x\}$, $B'=\{y\}$, $C=\{y\}$, and the Kolmogorov complexities
as before, the triangle inequality is not violated. In this case,
 $e_1(AB')=1 > e_1(AB)$
even though $AB' \subset AB$.
However it makes sense that if we add an element to a multiset of 
objects then a program to
go from any object in the new multiset to any other object should
be at least as long as a program to go from any object in the old multiset
to any other object. 
\end{remark}
The reasoning in the last sentence of the remark points the way to go:
the definition of $e(X)$ should be monotonic nondecreasing
in $|X|$ if we want $e$ to be a metric.
\begin{lemma}\label{lem.triangle}
Let $X,Y$ be multisets and $d$ be a distance
that satisfies the triangle 
inequality. 
If $Y \subseteq X$ then $d(Y) \leq d(X)$.
\end{lemma}
\begin{proof}
Let $A,B,C$ be multisets with $A,B \subseteq C$, and $d$ a distance
that satisfies the triangle inequality. Assume that the lemma is false and 
$d(C) < d(AB)$.
Let $D= C \setminus A$. It follows from the triangle inequality that
\[
d(AB) \leq d(AD)+d(DB).
\]
Since $AD=C$ this implies $d(AB) \leq d(C)+d(DB)$, and therefore
$d(C) \geq d(AB)$. But this contradicts the assumption.
\end{proof}

\begin{definition}
\rm
Let $X$ be
a multiset. Define the {\em normalized information
distance} (NID) for multisets by $e(X)=0$ for $|X|= 0,1$, and
for $|X|>1$ by
\begin{equation}\label{eq.multiples}
e(X)  = \max \left\{ \frac{\max_{x \in X} \{ K(X|x)\}}{\max_{x \in X}\{K(X \setminus \{x\})\}},
\max_{Y \subset X} \{ e(Y)\}\right\} .
\end{equation}
Here the left-hand term in the outer maximalization is denoted by $e_1(X)$
as in \eqref{eq.defp}.
\end{definition}
For $|X|=2$ the value of $e(X)$ is equivalently given in \eqref{eq.pairs}.
In this way, \eqref{eq.multiples} satisfies the property in Lemma~\ref{lem.triangle}:
If $X,Z$ are multisets and
$Z \subseteq X$ then $e(Z) \leq e(X)$. Therefore we can hope to prove
the triangle property for \eqref{eq.multiples}.
Instead of ``distance'' for multisets one can also use the term
``diameter.''  This does not change the acronym NID. Moreover, the
diameter of a pair of objects is the familiar distance.

\begin{theorem}\label{theo.01}
For every nonempty finite multiset $X$ we have $0 \leq e(X) \leq 1$
up to an additive term of $O(1/K)$ where $K=K(X)$.
\end{theorem}
\begin{proof}
By induction on $n=|X|$. 

{\em Base case:} $n=0,1$. 
The theorem is vacuously true for $n=0,1$. 

{\em Induction:} $n >1$. Assume that the lemma is true for the 
cases $0, \ldots , n-1$.
Let $|X| = n$. If $e(X) = \max_{Y \subset X}\{e(Y)\}$ then the lemma
holds by the inductive assumption since $|Y| < n$. Hence assume
that 
\[e(X)= \frac{\max_{x \in X} \{ K(X|x)\}}{\max_{x\in X}
\{K(X \setminus \{x\}) \}}.\] 
The numerator is at most the denominator up to an $O(1)$
additive term and the denominator is at most $K(X)$. The lemma is proven.
\end{proof}

For $n=2$ the definition of $e(X)$
is \eqref{eq.pairs}. The proof of the lemma for this case is
also in \cite{Li03}.

\begin{remark}\label{rem.ei}
\rm
The least value of $e(X)$ is reached if all occurrences of elements of $X$
are equal. In that case $0 \leq e(X) \leq  O(1/K(X))$.
The greatest value $e(X)=1+O(1/K(X))$ is reached if 
$\max_{x \in X} \{ K(X|x)\} \geq 
\max_{x\in X} \{K(X \setminus \{x\})+O(1)$. 
For example, if the selected conditional, say $y$, has no consequence
in the sense that $K(X|y)= K(X \setminus \{y\}|y)+O(1)=
K(X \setminus \{y\})+O(1)$. 
This happens if
$K(z|y)=K(z)$ for all $z \in X \setminus \{y\}$.

Another matter is the consequences of \eqref{eq.multiples}.
Use \eqref{eq.soi}
in the left-hand term in both the numerator and the denominator. Then
we obtain up to additive logarithmic terms in the numerator and denominator
\begin{align}\label{eq.increase}
\frac{\max_{x \in X} \{K(X|x)\}}{\max_{x \in X}\{K(X\setminus \{x\})\}}
&=
\frac{K(X)- \min_{x\in X}\{K(x)\}}{K(X)-\min_{x \in X} \{K(x|X\setminus \{x\})\}}
\\&= 1 - \frac{\min_{x\in X}\{K(x)\} - \min_{x \in X} \{K(x|X\setminus \{x\})\}}
{K(X)-\min_{x \in X} \{K(x|X\setminus \{x\})\}}.
\nonumber
\end{align}
This expression goes to 1
if both
\[ 
K(X) \rightarrow \infty , \; \frac{\min_{x\in X}\{K(x)\}}{K(X)} \rightarrow 0.
\]
This happens, for instance, if 
$|X|= n$, $\min_{x \in X}=0$, $K(X) > n^2$, and 
$n \rightarrow \infty$. 
Also in the case that $X=\{x,x,\ldots,x\}$ ($n$ copies of a fixed $x$) and
$n \rightarrow \infty$. Then $K(X)\rightarrow \infty$ 
and $\min_{x\in X}\{K(x)\}/K(X) \rightarrow 0$ with
$|X| \rightarrow \infty$. To consider another case, we have
$K(X) \rightarrow \infty$
and $\min_{x\in X}\{K(x)\}/K(X) \rightarrow 0$ if 
$\min_{x\in X}\{K(x)\} = o(K(X))$ and 
$\max_{x\in X}\{K(x)\}-\min_{x\in X}\{K(x)\} \rightarrow \infty$,
that is, if $X$ consists of at least two elements and gap between
the minimum Kolmogorov
complexity and the maximum Kolmogorov complexity of 
the elements grows to infinity when $K(X) \rightarrow \infty$.
\end{remark}
\begin{remark}\label{rem.defp}
\rm
Heuristically the partitioning algorithm described in Section~\ref{sect.app}
 approaches the question
when to use the left-hand term of \eqref{eq.multiples}. But 
we can analyze directly under what conditions there is a $Y \subset X$ such that
$e(Y) > e(X)$.
Without loss of generality we can assume that in that case
the left-hand term of \eqref{eq.multiples} for $Y$
is greater than the left-hand term of \eqref{eq.multiples} 
for $X$ (that is, $e_1(Y)>e_1(X)$ with $e_1$ according to \eqref{eq.defp}).
This means that 
\begin{equation}\label{eq.ee}
\frac{K(Y) - \min_{x\in Y}\{K(x)\}}
{\max_{x\in Y}\{K(Y \setminus \{x\})\}}
>
\frac{K(X)-\min_{x\in X}\{K(x)\}}
{\max_{x\in X}\{K(X \setminus \{x\})\}}, 
\end{equation}
ignoring logarithmic additive terms.
Take the example of Remark~\ref{rem.ex}. Let $X=\{x,y,y\}$ and $Y=\{x,y\}$.
Then $Y \subset X$. The left-hand side of \eqref{eq.ee} equals 1
and the right-hand side equals $\approx 1/2$. 
In this case $e_1(Y) > e_1(X)$ with $e_1$
according to \eqref{eq.defp} and as we have seen the triangle inequality 
does not hold for $e_1$.
\end{remark}

\begin{theorem}\label{theo.nmetric}
The function $e$ as in \eqref{eq.multiples} is
a metric up to an 
additive $O((\log K)/K)$ term in the respective metric (in)equalities, where $K$
is the largest Kolmogorov complexity involved the (in)equality.
\end{theorem}
\begin{proof}
The quantity $e(X)$ satisfies positive definiteness and symmetry
up to an $O((\log K(X))/K(X))$ additive term, 
as follows directly from the definition of $e(X)$ in \eqref{eq.multiples}. 
It remains to prove the triangle inequality:
 
Let $X,Y,Z$ be multisets. Then,
$e(XY) \leq e(XZ)+e(ZY)$ within an additive term of
$O((\log K)/K)$ where $K= \max\{K(X),K(Y),K(Z)\}$.
The proof proceeds by induction on $n=|XY|$. 

{\em Base Case:} $n=0,1$. These cases are vacuously true. 

{\em Induction $n >1$}. 
Assume that the lemma is true for the cases $0, \ldots , n-1$.
Let $|XY| = n$. If $e(XY) = \max_{Z \subset XY}\{e(Z)\}$ then the lemma
holds by the inductive assumption since $|Z| < n$. 
Hence assume
that 
\[
e(XY)=e_1(XY) = \frac{K(XY|x_{XY})}{
K(XY \setminus \{x_{xy}\})},
\]
where we let $x_u$ and $x_V$ denote the elements that reach the maximum in
 $K(U \setminus \{x_{u}\})=\max_{x \in U}\{K(U \setminus \{x\})\}$,
and $K(V|x_{V})=\max_{x \in V}\{K(V|x)\}$.

\begin{claim}\label{claim.tr}
\rm
Let $X,Y,Z$ be multisets (finite, but possibly empty).
$K(XYZ| x_{XYZ}) \leq  K(XZ | x_{XZ}) + K(ZY | x_{ZY})$ up to an
additive $O(\log K)$ term, where $K=K(X)+K(Y)+K(Z)$.
\end{claim}
\begin{proof}
If one or more of $X,Y,Z$ equal $\emptyset$ the theorem holds trivially.
Therefore, assume neither of $X,Y,Z$ equals $\emptyset$.
By Theorem~\ref{theo.metric} we have that $E_{\max}$ and hence
$K(XY|x_{XY})$ is a metric up to an $O(\log K)$ additive term.
In particular, the triangle inequality is satisfied by Claim~\ref{claim.metric}:
$K(XY | x_{XY}) \leq K(XZ | x_{XZ}) + K(ZY | x_{ZY})$ for multisets
$X,Y,Z$ up to an additive term of $O(\log K)$. Thus with
$X' = XZ$ and $Y' = ZY$ we have
$K(X'Y' | x_{X'Y'}) \leq K(X'Z | x_{X'Z}) + K(ZY' | x_{ZY'})$
up to the logarithmic additive term.
Writing this out
$K(XZZY | x_{XZZY}) \leq K(XZZ | x_{XZZ}) + K(ZYZ | x_{ZYZ})$ or
$K(XYZ | x_{XYZ}) \leq  K(XZ | x_{XZ}) + 
K(ZY | x_{ZY})$ up to an additive term of $O( \log K )$.
\end{proof}
Now consider the following inequalities (where possibly one or more of
$X,Y,Z$ equal $\emptyset$):
\begin{align}\label{eq.si}
e_1(XYZ) &=\frac{K(XYZ|x_{XYZ})}{K(XYZ \setminus \{x_{xyz} \})}
\\& \leq \frac{K(XZ|x_{XZ})}{K(XYZ \setminus \{x_{xyz}\})}
+ \frac{K(ZY|x_{ZY})}{K(XYZ \setminus \{x_{xyz}\})}
\nonumber
\\& \leq \frac{K(XZ|x_{XZ})}{K(XZ \setminus \{x_{xz} \})}
+ \frac{K(ZY|x_{ZY})}{K(ZY \setminus \{x_{zy} \})}
\nonumber
\\& = e_1(XZ)+e_1(ZY),
\nonumber
\end{align}
up to a $O((\log K)/K)$ additive term. The first inequality is
Claim~\ref{claim.tr} (by this inequality the denominator is unchanged);
the second inequality
follows from $K(XYZ\setminus \{x_{xyz}\}) \geq K(XZ\setminus \{x_{xz}\})$
and $K(XYZ\setminus \{x_{xyz}\}) \geq K(ZY\setminus \{x_{zy}\})$ using the
principle that $K(uv) \geq K(u)+O(1)$, reducing both denominators
and increasing the sum of the quotients (by this inequality
the numerators are unchanged); the third inequality follows by definition
\eqref{eq.multiples}. 

By definition \eqref{eq.multiples} a multiset $XYZ$ has 
$e(XYZ)=e_1(XYZ)$ or it contains a proper submultiset $U$ 
such that $e(U)=e_1(U)=e(XYZ)$. This $U \subset XYZ$ is the multiset (if it exists)
that achieves the maximum 
in the right-hand term of the outer maximalization of $e(XYZ)$ in \eqref{eq.multiples}.

Assume $U$ exists.
Denote $X'=X\bigcap U$, $Y'=Y \bigcap U$, and $Z'=Z\bigcap U$. 
Then \eqref{eq.si} holds with $X'$ substituted for $X$, $Y'$ substituted for
$Y$, and $Z'$ substituted for $Z$. 
Since $e(U)=e_1(U)$ and $e(XY) \leq e(U)$ we have
$e(XY) \leq e_1(X'Z')+e_1(Z'Y')$ up to a $O((\log K)/K)$ additive term. 

Assume $U$ does not exist. Then $e(XY) \leq e(XYZ)=e_1(XYZ)$.
By \eqref{eq.si} we have $e(XY) \leq e_1(XZ)+e_1(ZY)$ up to a 
$O((\log K)/K)$ additive term.

By the monotonicity property of \eqref{eq.multiples} and since
$X'Z' \subseteq XZ$ and $Z'Y' \subseteq ZY$ we have
$e(XZ) \geq e_1(X'Z'), e_1(XZ)$ and $e(ZY) \geq e_1(Z'Y'), e_1(ZY)$.
Therefore, $e(XY) \leq e(XZ)+e(ZY)$ 
up to an $O((\log K)/K)$ additive term.
\end{proof}

\begin{remark}
\rm
The definition of $e(XY)$ with $|XY|=2$
is \eqref{eq.pairs}. The proof of the lemma for this case is in \cite{Li03}.
The proof above is simpler and more
elementary for a more general case than the one in \cite{Li03}.

\end{remark}
By Theorems~\ref{theo.01} and \ref{theo.nmetric} the distance according
to \eqref{eq.multiples} is a metric with values in $[0,1]$ up to some
ignorable additive terms..

\section{Compression Distance for Multisets}\label{sect.cd}
We develop the compression-based equivalence of the Kolmogorov complexity
based theory in the preceding sections. This is similar to 
what happened in \cite{CV04}
for the case $|X|=2$. We assume that the notion of the real-world compressor
$G$ used in the sequel is ``normal'' in the sense of \cite{CV04}.
\begin{definition}\label{def.GX}
\rm
By $G(x)$ we mean the length of string $x$ when compressed by $G$.
Consider a multiset $X$ as a string consisting of the concatenated 
strings of its members ordered length-increasing lexicographic with a means
to tell the constituent elements apart. Thus we can
write $G(X)$.
\end{definition}

Let $X=\{x_1, \ldots , x_m\}$. The information distance $E_{\max}(X)$ can
be rewritten as
\begin{equation}\label{eq.nom}
\max \{ K(X)-K(x_1), \ldots ,  K(X)-K(x_m) \},
\end{equation}
within logarithmic
additive precision, by \eqref{eq.soi}.
The term $K(X)$ represents the length of the shortest
program for $X$. The order of the members of $X$ makes only a
small difference; block-coding based compressors are symmetric
almost by definition, and experiments with various stream-based
compressors (gzip, PPMZ) show only small
deviations from symmetry. 

Approximation of $E_{\max}(X)$
by a compressor $G$
is straightforward: it is
\begin{equation}\label{eq.EG}
E_{G, \max}(X)=\max \{ G(X)-G(x_1), \ldots ,  G(X)-G(x_m) \}
= G(X)- \min_{x \in X} \{G(x)\}.
\end{equation}
We need to show it is an admissible distance
and a metric.
\begin{lemma}\label{lem.ad}
If $G$ is a normal compressor, then $E_{G, \max}(X)$ is an admissible 
distance.
\end{lemma}
\begin{proof}
For $E_{G, \max}(X)$ to be an admissible distance it must satisfy
the density requirement \eqref{eq.density} and be upper semicomputable.
Since the length $G(x)$ is computable it is a fortiori upper semicomputable.
The density requirement \eqref{eq.density} is equivalent
to the Kraft inequality \cite{Kr49} and states in fact that for every string
$x$ the set of $E_{G, \max}(X)$ is a prefix-free code for the $X$'s
containing $x$. According to \eqref{eq.EG} we have for every $x\in X$:
$E_{G, \max}(X) \geq  G(X)-G(x) \geq G(X \setminus \{x\})$. 
Hence, $2^{-E_{G, \max}(X)}
\leq 2^{-G(X\setminus \{x\})}$ and therefore 
\[
\sum_{X:x \in X} 2^{-E_{G, \max}(X)} \leq \sum_{X:x \in X}
2^{-G(X \setminus \{x\})}.
\] 
A compressor $G$ compresses strings into a uniquely
decodable code (it must satisfy the unique decompression property) and 
therefore the length set of the compressed strings must satisfy the 
Kraft inequality \cite{Mc56}. Then, for every $x$ 
the compressed code for the multisets $X \setminus \{x\}$ must satisfy
this inequality. Hence the right-hand side of above displayed inequality 
is at most 1.
\end{proof}

\begin{lemma}\label{lem.mg}
If $G$ is a normal compressor, then $E_{G, \max}(X))$ is a metric
with the metric (in)equalities satisfied up to logarithmic additive precision. 
\end{lemma}
\begin{proof}
Let $X,Y,Z$ be multisets with at most $m$ members of length at
most $n$.
The positive definiteness and the symmetry property hold clearly
up to an $O(\log G(X)))$ additive term.
Only the triangular inequality is nonobvious. 
For every compressor $G$ we have $G(XY) \leq G(X)+G(Y)$ up to an additive
$O(\log (G(X)+G(Y)))$ term, otherwise we
obtain a better compression by dividing the string to be compressed.
(This also follows from the distributivity property of normal compressors.)
By the monotonicity property $G(X) \leq G(XZ)$ and $G(Y) \leq G(YZ)$ up to an
$O(\log (G(X)+G(Y)))$ or $O(\log (G(Y)+G(Z)))$ 
additive term, respectively. Therefore,
$G(XY) \leq G(XZ)+G(ZY)$ up to an $O(\log (G(X)+G(Y)+G(Z)))$ additive term.
\end{proof}

\section{Normalized Compression Distance for Multisets}\label{sect.ncd}
Let $X$ be a multiset.
The normalized version of 
$e(X)$ using the compressor $G$ based approximation of the normalized
information distance for multisets  \eqref{eq.multiples}, is called the
{\em normalized compression distance} (NCD) for multisets: 
$NCD(X)=0$ for $|X|=0,1$; if $|X| \geq 2$ then
\begin{equation}\label{eq.ncd}
NCD(X) = \max \left\{\frac{G(X)- \min_{x \in X} \{G(x)\}}
{\max_{x \in X}\{G(X\setminus \{x\}\}}, \max_{Y \subset X} \{NCD(Y)\}\right\}.
\end{equation}
This $NCD$
is the main concept of this work. It is the real-world
version of the ideal notion of normalized information distance
NID for multisets in \eqref{eq.multiples}.
As mentioned before, instead of ``distance'' for multisets it 
one can use also the term
``diameter.''  This does not change the acronym NCD.

\begin{remark}
\rm
In practice,
the NCD is a non-negative number
$0 \leq  r \leq 1 + \epsilon$ representing how
different the two files are. Smaller numbers represent more similar files.
The $\epsilon$ in the upper bound is due to
imperfections in our compression techniques,
but for most standard compression algorithms one is unlikely
to see an $\epsilon$ above 0.1 (in our experiments \texttt{gzip}
and \texttt{bzip2} achieved
NCD's above 1, but \texttt{PPMZ} always had NCD at most 1).
\end{remark}
\begin{theorem}
If the compressor is normal, then the NCD for multisets
is a normalized admissible distance and
satisfies the metric (in)equalities up to an ignorable
additive term, that is, it is a similarity metric.
\end{theorem}
\begin{proof}
The NCD \eqref{eq.ncd} is a normalized admissible distance by
Lemma~\ref{lem.ad}. It is normalized to $[0,1]$ up to an additive term
of $O((\log G)/G)$ with $G=G(X)$ as we can see from
the formula \eqref{eq.ncd} and Theorem~\ref{theo.01} 
with $G$ substituted for $K$
throughout. We next show it is a metric. 

We must have that $NCD(X)= 0$ up to negligible error for 
a  normal compressor $G$ if
$X$ consists of equal members. The idempotency property of a normal
compressor is up to an additive term of $O(\log G(X))$. Hence
the positive definiteness of $NCD(X)$ is satisfied up to an additive
term of $O((\log G(X))/G(X))$.
The order of the members of $X$ is assumed to be length-increasing 
lexicographic. Therefore it is symmetric up to an additive term
of $O((\log G(X))/G(X))$. It remains to show the triangle
inequality $NCD(XY) \leq NCD(XZ)+NCD(ZY)$ up to an additive term
of $O((\log G)/G)$ where $G=G(X)+G(Y)+G(Z)$. 
We do this by induction on $n=|XY|$ where $X,Y$ are 
multisets. 

{\em Base case:} $n=0,1$. The triangle property is vacuously satisfied. 

{\em Induction:} $n>1$. Assume the triangle property is satisfied 
for the cases $1, \ldots,  n-1$.
We prove it for $|XY|=n$.
If $NCD(XY)=NCD(Z)$ for some $Z \subset XY$ then 
the case follows from the inductive argument. Therefore, $NCD(XY)$
is the first term in the outer maximization of \eqref{eq.ncd}.
Write $G(XY|x_{XY})=G(XY)- \min_{x \in XY} \{G(x) \}$
and $G(XY \setminus \{x_{xy}\})= \max_{x \in XY}\{G(XY)\setminus \{x\}\}$
and similar for $XZ,YZ,XYZ$. Following the induction case of
the triangle inequality  in the proof of
Theorem~\ref{theo.nmetric}, using Lemma~\ref{lem.mg} for the metricity
of $E_{G,\max}$ wherever Theorem~\ref{theo.metric} is used to 
assert the metricity
of $E_{\max}$, and substitute $G$ for $K$ in the remainder.
This completes the proof. That is,
for every multiset $Z$ we have
\[
NCD(XY) \leq NCD(XZ)+NCD(ZY),
\]
up to an additive term of $O((\log G) /G)$.
\end{proof}

For $|XY|=2$ the triangle property is also proved in \cite{CV04}. The proof
above is simpler and more elementary.

\section{Computing the Normalized Compression Distance for Multisets}
\label{sect.comp}
Define
\begin{equation}\label{eq.ncd1}
NCD_1(X) = \frac{G(X)- \min_{x \in X} \{G(x)\}}
{\max_{x \in X}\{G(X\setminus \{x\}\}},
\end{equation}
the first term of \eqref{eq.ncd}.
Assume we want to compute
$NCD(X)$ and $|X|=n \geq 2$.
In practice it seems that one can do no better than the 
following (initialized with $M_i=0$ for $i \geq 1$):

{\bf for} $i=2,\ldots, n$ 

{\bf do}
$M_i := \max\{\max_{Y}\{NCD_1(Y): Y\subset X, \: |Y| = i\}, M_{i-1}\}$ using 
\eqref{eq.ncd1} {\bf od}

$NCD(X) := M_n$

However, this process involves evaluating the $NCD$'s of the
entire powerset of $X$ requiring at least order $2^n$ time. 
\begin{theorem}\label{theo.fast}
Let $X$ be a multiset and $n=|X|$.
There is a heuristic algorithm to approximate $NCD(X)$ from below
in $O(n^2)$ computations of $G(Y)$ with $Y \subseteq X$. (Assuming 
every $x \in Y$ to be a binary string, $|x|\leq m$, and  $G$ compresses 
in linear time then $G(Y)$ is computed in $O(nm)$ time.) 
\end{theorem}
\begin{proof}
We use the analysis in Remark~\ref{rem.defp} and in particular
the inequality
\eqref{eq.ee}. We ignore logarithmic additive terms.
We approximate $NCD(X)$ from below by
$\max_{Y \subseteq X} \{NCD_1(Y)\}$ for a sequence of $n-2$ properly
nested $Y$'s of decreasing cardinality. 
That is, in the computation we set the value of $NCD(X)$ to 
$NCD_1(X)$ unless there is one of these $Y$'s such that
$NCD_1(X) < NCD_1(Y)$ in which case we set the value of $NCD(X)$ to $NCD_1(Y)$.
How do we choose this sequence of $Y$'s?
\begin{claim}\label{claim.1}
Let $Y \subset X$ and 
$G(X) - \min_{x\in X}\{G(x)\}   - \max_{x\in X}\{G(X \setminus \{x\})\}
< G(Y) - \min_{x\in Y}\{G(x)\} -\max_{x\in Y}\{G(Y \setminus \{x\})\}$. 
Then, $NCD_1(X) < NCD_1(Y)$. 
\end{claim}
\begin{proof}
We first show that
$\max_{x\in Y}\{G(Y \setminus \{x\})\} 
\leq \max_{x\in X}\{G(X \setminus \{x\})\}$. 
Let $G(Y \setminus \{y\}) = \max_{x\in Y}\{G(Y \setminus \{x\})\}$. Since
$Y \subset X$ we have 
$G(Y \setminus \{y\}) \leq G(X \setminus \{y\}) 
\leq \max_{x\in X}\{G(X \setminus \{x\})\}$.

We next show that if $a-b < c-d$ and $d \leq b$ then $a/b < c/d$. 
Namely, dividing the first inequality by $b$ we obtain 
$a/b -b/b < (c - d)/b \leq (c-d)/d$.
Hence, $a/b < c/d$.

Setting $a=G(X) - \min_{x\in X}\{G(x)\}$, 
$b=\max_{x\in X}\{G(X \setminus \{x\})\}$, 
$c=G(Y)-\min_{x\in Y}\{G(x)\}$,
and $d=\max_{x\in Y}\{G(Y \setminus \{x\})\}$, 
the above shows that the claim holds.
\end{proof}

Claim~\ref{claim.1} states that the only candidates $Y$ ($Y \subset X$) for
$NCD_1(Y) > NCD_1(X)$ are the $Y$
such that $G(X)- \min_{x\in X}\{G(x)\} - \max_{x\in X}\{G(X \setminus \{x\})\} 
< G(Y) - \min_{x\in Y}\{G(x)\} -
\max_{x\in Y}\{G(Y \setminus \{x\})\}$. 

For example,
let $X=\{x_1,x_2, \ldots, x_n\}$, $|Y|=2$, 
$G(X)= \max_{x\in X}\{G(X \setminus \{x\})\}$
(for instance $x_1=x_2$), and $\min_{x\in X}\{G(x)\} > 0$. Clearly,
$G(Y) - \max_{x\in Y}\{G(Y \setminus \{x\})\}
= G(Y)- \max_{x \in Y}\{G(x)\}
=\min_{x \in Y}\{G(x)\}$. Then, $0=G(X)- \max_{x\in X}\{G(X \setminus \{x\})\} 
< G(Y) - \max_{x\in Y}\{G(Y \setminus \{x\})\} +\min_{x\in X}\{G(x)\} -
\min_{x\in Y}\{G(x)\} = \min_{x \in Y}\{G(x)\} +\min_{x\in X}\{G(x)\} -
\min_{x\in Y}\{G(x)\} = \min_{x\in X}\{G(x)\}$. 

Hence for $Y \subset X$, if 
$G(X)-\max_{x\in X}\{G(X \setminus \{x\})\}$ is smaller than
$G(Y)-G(\max_{x\in Y}\{G(Y \setminus \{x\})\} +\min_{x\in X}\{G(x)\} -
\min_{x\in Y}\{G(x)\}$ then $NCD_1(Y) > NCD_1(X)$. Note that
if the $x$ that maximizes $\max_{x\in X}\{G(X \setminus \{x\})\}$ 
is not the $x$ that minimizes  $\min_{x\in X}\{G(x)\}$
then $\min_{x\in X}\{G(x)\} - \min_{x\in Y}\{G(x)\} =0$, otherwise
$\min_{x\in X}\{G(x)\} - \min_{x\in Y}\{G(x)\} < 0$. 

Removing the element that minimizes 
$G(X)-\max_{x\in X}\{G(X \setminus \{x\})\}$ may make
the elements of $Y$ more dissimilar and therefore 
increase $G(Y)-G(\max_{x\in Y}\{G(Y \setminus \{x\})\}$.
Iterating this process may make the elements of the resulting sets ever
more dissimilar, until the associated $NCD_1$ declines due to
decreasing cardinality.   

Therefore, we come to the following heuristic.
Let $X=\{x_1, \ldots , x_n\}$. Compute 
$$G(X) - \max_{x\in X}\{G(X \setminus \{x\})\}.$$
Let $I$ be the index $i$ for which the maximum in the second term is reached. 
Set $Y_1=X \setminus \{x_I\}$. Repeat this process with $Y_1$ instead of $X$ to
obtain $Y_2$, and so on. The result is $Y_{n-2} \subset \ldots \subset Y_1 
\subset Y_0$ with $Y_0=X$ and $|Y_{n-2}|=2$. Set
$NCD(X)=\max_{0 \leq i \leq n-2} \{NCD_1(Y_i) \}$.
The whole process to compute this heuristic to approximate
$NCD(X)$ from below takes $O(n^2)$ steps where a step involves 
compressing a subset of $X$ in $O(nm)$ time. 
\end{proof}

The inequality $NCD_1(Y) > NCD_1(X)$ for $Y \subset X$ in the form of $e_1(Y) > e_1(X)$ with
$|Y|=2$ occurs in the example of Remark~\ref{rem.defp}.
Here $e_1$ is according to \eqref{eq.defp}, that is, the left term in
\eqref{eq.multiples} defining $e(X)$.

\section{How To Apply The NCD}\label{sect.ancd}

The applications in Section~\ref{sect.app} concern classifications.
For these purposes, it makes no sense to compute the $NCD$, but instead 
we consider the change in $NCD_1$ for each multiset that we are classifying against 
with and without the element that is being classified. To compare these changes we
require as much discriminatory power as possible. 

\subsection{Theory}
Remark~\ref{rem.ei} shows that $NCD_1(X) \rightarrow 1$ for 
$G(X) \rightarrow \infty$
and $\min_{x \in X} \{G(x)\}/G(X) \rightarrow 0$ (in the form of $e_1(X)$).
While possibly $NCD_1(X) < NCD(X)$, 
for some $X' \supset X$ we have $NCD_1(X')=NCD(X')$. 
In general we believe that the nondecreasing
of $NCD(X)$ with larger $X$ according to Lemma~\ref{lem.triangle} is just required
to make the $NCD$ metric. This smoothes the $NCD_1$ function according to \eqref{eq.ncd1}
to a nondecreasing function for increasing cardinality arguments. 
However, in the smoothing
one obliterates differences that obviously enhance discriminatory power.

\subsection{Practice}
Suppose we want to classify $x$ as 
belonging to one of the classes represented by multisets $A,B, \ldots, Z$.
Our method is to consider $NCD(A \bigcup \{x\})-NCD(A)$, and similar for classes
represented by $B, \ldots , Z$, and then to select the least difference. 
However, this difference is always greater or equal to 0 by Lemma~\ref{lem.triangle}.
If we look at $NCD_1(A \bigcup \{x\})-NCD_1(A)$ then the difference may be negative,
zero, or positive and possibly greater in absolute value. This gives larger
discriminatory power in the classes selection.

Another reason is as follows.
Let $Y_0$ be as in the proof of Theorem~\ref{theo.fast}.
For the handwritten digit recognition application in Section~\ref{sect.NIST} we 
computed $Y_0$ for digits $1,2, \ldots ,9,0$. The values were 0.9845, 0.9681, 0.9911,
0.9863, 0.9814, 0.9939, 0.9942, 0.9951,0.992, 0.9796. But
$\max_{0 \leq i \leq n-2} \{NCD_1(Y_i) \} =0.9953$ for the class of digit 1 where the
maximum is reached for index $i=21$.
Thus $NCD(A \bigcup \{x\})-NCD_1(A \bigcup \{x\}) = 0.0108$ for this class with the $NCD$
computed according to Theorem~\ref{theo.fast}. We know that $NCD(A) \leq NCD(A \bigcup \{x\})$ 
because $A \subset A \bigcup \{x\}$. Computing $NCD(A)$ similarly as 
$NCD(A \bigcup \{x\})-NCD_1(A \bigcup \{x\})$
may yield $NCD(A) = NCD(A \bigcup \{x\})$ because 
$Y_{j}$ for $A \bigcup \{x\}$ 
may be the same multiset as $Y_{j-1}$ for $A$ for some $1 \leq j \leq 21$. 
This has nothing to do with the element $x$
we try to classify. The same may happen to $NCD(B \bigcup \{x\})-NCD(B)$, and so on.

\subsection{Kolmogorov Complexity of Natural Data}
The Kolmogorov
complexity of a file is a lower bound on the length
of the ultimate compressed version of that
file.
Above we approximate the Kolmogorov complexities involved by
a real-world compressor $G$. Since the Kolmogorov complexity is incomputable,
in the approximation we never know how close we are to it. However,
we assume that the natural data we are dealing with contain no complicated
 mathematical
constructs like $\pi=3.1415 \ldots$ or Universal Turing machines.
In fact, we assume that the natural data we are dealing with contains
only effective regularities that a good compressor like $G$ finds. Under those
assumptions the Kolmogorov complexity $K(x)$ of object $x$ is not much smaller than
the length of the compressed version $G(x)$ of the object.

\subsection{Partition Algorithm}
Section \ref{sect.NIST} describes an algorithm that we developed to 
partition data for classification in cases where the classes are not 
well separated so that there are no subsets of a class with separation 
larger than that of the smallest inter-class separation, 
a heuristic that we have found to work well in practice. 

\section{Applications}\label{sect.app}

We detail preliminary results using the new NCD for multisets.
The NCD for pairs as originally defined \cite{CV04} has been applied in a wide 
range of application domains. In \cite{Ke04} a close relative was compared to 
every time series distance measure published in the decade preceding 2004 
from all of the major data analysis conferences
 and found to outperform all other distances aside from the Euclidean distance 
with which it was competitive. The NCD for pairs has also been applied in 
biological applications to analyze the results of segmentation 
and tracking of proliferating cells and organelles \cite{Co09,Co10,Wi12}. 
The NCD is unique in allowing multidimensional
time sequence data to be compared directly, with no need for alignment or averaging. 

Here, we compare the performance of the proposed NCD for multisets 
to that of 
a previous application of the NCD for pairs for predicting retinal progenitor cell
(RPC) fate outcomes from the segmentation and tracking results from live 
cell imaging \cite{Co10}. We also apply the proposed NCD to a synthetic data 
set previously analyzed with the pairwise NCD \cite{Co09}. Finally, we apply the proposed NCD 
for multisets to the classification of handwritten digits, an application
that was previously evaluated using the pairwise NCD in \cite{CV04}.

\subsection{Retinal Progenitor Cell Fate Prediction}\label{sect.rpc}
In \cite{Co10}, long-term time-lapse image sequences showing rat RPCs were 
analyzed using automated segmentation and tracking algorithms. 
Images were captured every five minutes of the RPCs for a period of 9--13 days. 
Up to 100 image sequences may be captured simultaneously in this manner using 
a microscope with a mechanized stage. For an example see Figure~\ref{fig.rpc}.
\begin{figure}[htb]
\begin{center}
\includegraphics[width=4.5in]{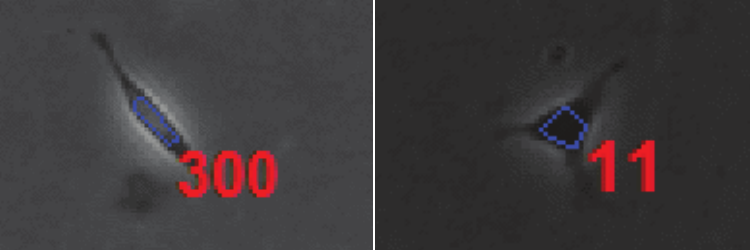}
\end{center}
\caption{Example frames from two retinal progenitor cell 
(RPC) image sequences showing segmentation (blue lines) and tracking 
(red lines) results. The type of cells the RPCs will eventually produce 
can be predicted by analyzing the multidimensional time sequence data 
obtained from the segmentation and tracking results. 
The NCD for multisets significantly improves the accuracy of the predictions.}
\label{fig.rpc}
\end{figure}
At the conclusion of the experiment, 
the ``fate'' of the offspring produced by each RPC was determined using 
a combination of cell morphology and specific cell-type fluorescent markers 
for the four different retinal cell types produced from embryonic day 20 rat 
RPCs \cite{Ca03}. At the conclusion of the imaging, automated segmentation and 
tracking algorithms \cite{Wi11} were applied to extract the time course of 
features for each cell. These automated segmentation and tracking algorithms 
extract a time course of feature data for each stem cell at a five-minute 
temporal resolution, showing the patterns of cellular motion and morphology 
over the lifetime of the cell. Specifically, the segmentation and tracking results
consisted of a 6-dimensional time sequence feature vector incorporating 
two-dimensional motion $(\Delta x, \Delta y)$, as well as the direction of motion, 
total distance travelled, cellular size or area (in pixels) and a measure of 
eccentricity on $[0,1]$ (0 being linear, 1 being circular shape). 
The time sequence feature vectors for each of the cells are of different 
length and are not aligned. The results from the segmentation and 
tracking algorithms were then analyzed as follows. 

The original analysis of the RPC segmentation and tracking results used a 
multiresolution semi-supervised spectral analysis based on the originally 
formulated pairwise NCD. An ensemble of distance matrices consisting of pairwise 
NCDs between quantized time sequence feature vectors of individual cells 
is generated for different feature subsets $f$ and different numbers of 
quantization symbols $n$ for the numerical time sequence data. The fully automatic
quantization of the numeric time sequence data is described in \cite{Co09}. All 
subsets of the 6-dimensional feature vector were included, although it is 
possible to use non-exhaustive feature subset selection methods such as forward 
floating search, as described in \cite{Co09}. Each distance matrix is then 
normalized as described in \cite{Co10}, and the eigenvectors and eigenvalues of 
the normalized matrix are computed. These eigenvectors are stacked and ordered 
by the magnitude of the corresponding eigenvalues to form the columns of a new 
``spectral'' matrix. The spectral matrix is a square matrix, of the same dimension
$N$ as the number of stem cells being analyzed. The spectral matrix has the 
important property that the $i$th row of the matrix is a point in 
${\mathbb R}^N$ (${\mathbb R}$ is the set of real numbers) 
that corresponds to the 
quantized feature vectors for the $i$th stem cell. If we consider only the 
first $k$ columns, giving a spectral matrix of dimension $N\times k$, and run a 
K-Means clustering algorithm, this yields the well-known spectral
K-Means algorithm \cite{Ka03}. If we have known outcomes for any of the objects 
that were compared using the pairwise NCD, then we can formulate a semi-supervised
spectral learning algorithm by running for example nearest neighbors or decision 
tree classifiers on the rows of the spectral matrix. This was the approach adopted
in \cite{Co10}. 

In the original analysis, three different sets of known outcomes were considered. 
First, a group of 72 cells were analyzed to identify cells that would self-renew 
(19 cells), producing additional progenitors and cells that would terminally 
differentiate (53 cells), producing two retinal neurons. Next, a group of 86 cells
were considered on the question of whether they would produce two photoreceptor 
neurons after division (52 cells), or whether they would produce some other 
combination of retinal neurons (34 cells). Finally, 78 cells were analyzed to 
determine the specific combination of retinal neurons they would produce, 
including 52 cells that produce two photoreceptor neurons, 10 cells that produce 
a photoreceptor and bipolar neuron, and 16 cells that produced a photoreceptor 
neuron and an amacrine cell. Confidence intervals are computed for the classification results by treating the classification accuracy as a normally distributed random variable, and using the sample size of the classifier together with the normal cumulative distribution function (CDF) to estimate the region corresponding to a fixed percentage of the distribution \cite[pp. 147-149]{Witten2005}. For the terminal versus self-renewing question, 
99\% accuracy was achieved in prediction using a spectral nearest neighbor 
classifier, with a 95\% confidence interval of [0.93, 1.0]. In the sequel, we will list the 95\% confidence interval in square brackets following each reported classification accuracy. For the two photoreceptor versus other combination question, 87\% 
accuracy [0.78, 0.93] was achieved using a spectral decision tree classifier. 
Finally, for the specific combination of retinal neurons 83\% accuracy [0.73, 0.9] was
achieved also using a spectral decision tree classifier.

Classification using the newly proposed NCD \eqref{eq.multiples} is much more 
straightforward and leads to significantly better results. 
Given multisets $A$ and $B$, each consisting of cells having a given fate, 
and a cell $x$ with unknown fate, we proceed as follows. We assign $x$ to 
whichever multiset has its distance (more picturesque ``diameter'')
 increased the least with the addition of 
$x$. In other words, if 
\begin{equation}\label{eq.class}
NCD_1(Ax)-NCD_1(A)<NCD_1(Bx)-NCD_1(B),
\end{equation}
we assign $x$ to multiset $A$, else we assign $x$ to multiset $B$. (The
notation $Xx$ is shorthand for the multiset $X$ with one occurrence of $x$ added.) Note that for classification purposes we consider the impact of element $x$ on the $NCD_1$ \eqref{eq.ncd1}  only and do not evaluate the full NCD for classification. We use the $NCD_1$ in \eqref{eq.class} rather than the $NCD$ because the $NCD_1$ has the ability to decrease when element $x$ contains redundant information with respect to multiset $A$. See also the reasons
in Section~\ref{sect.ancd}. 

The classification accuracy improved considerably using the newly proposed NCD
for multisets. For the terminal versus self-renewing question, we achieved 
100\% accuracy in prediction [0.95,1.0] compared to 99\% accuracy [0.93,1.0] for the multiresolution 
spectral pairwise NCD. For the two photoreceptor versus other combination 
question, we also achieved 100\% accuracy [0.95,1.0] compared to 87\% [0.78,0.93]. Finally, for the 
specific combination of retinal neurons we achieved 92\% accuracy [0.84,0.96] compared to 
83\% [0.73,0.9] with the previous method.

\subsection{Synthetic  Data}\label{sect.synth}
In \cite{Co09}, an approach was developed that used the pairwise NCD
to compute a concise and meaningful summarization of the results 
of automated segmentation and tracking algorithms applied to biological image sequence
 data obtained from live cell and tissue microscopy.
A synthetic or simulated data set was analyzed
using a method that incorporated the pairwise NCD.
allowing precise control over differences between objects
within and across image sequences. The features for the
synthetic data set consisted of a
23-dimensional feature vector. The seven features relating
to 3-D cell motion and growth were modeled as described
below, the remaining 16 features were set to random values.
Cell motility was based on a ???run-and-tumble???
model similar to the motion of bacteria. This
consists of periods of rapid directed movement followed by
a period of random undirected motion. Cell lifespan was modeled as a gamma
distributed random variable with shape parameter 50 and scale parameter 10. 
Once a cell reaches its
lifespan it undergoes cell division, producing two new
cells, or, if a predetermined population limit has been
reached, the cell undergoes apoptosis, or dies. The final aspect of the model
was cell size. The initial cell radius,
denoted ${r}_{0}$, is a gamma-distributed random variable with shape 
parameter 200 and scale parameter 0.05. The
cells growth rate is labeled $\upsilon$. At the end of its lifespan, the
cell doubles its radius. The radius at time t is given by
\[r(t)={{r}_{0}}+{{r}_{0}}\cdot {{\left( \frac{t-{{t}_{0}}}{lifespan} \right)}^{\upsilon }}\]
In the original analysis, two different populations were  simulated, one population having an
$\upsilon$  value of  3,  the second having  an $\upsilon$  value of 0.9. 

The data was originally analyzed using a multiresolution representation of the time sequence data along with 
feature subset selection. Here we repeat the analysis for a population of 656 simulated cells, with between
228 and 280 time values for each 23 dimensional feature vector.
This data was analyzed using a minimum distance supervised classifier with both the original pairwise and 
the proposed NCD for  multisets. Omitting the feature subset selection step and incorporating the entire 23   
dimensional feature vector,  the pairwise NCD was 57\% correct [0.53,0.61] the classifying the data, 
measured by  leave-one-out cross validation. Using  NCD  for multisets, we achieved 91\% correct [0.89,.93] classification,
a significant improvement.
When  a feature subset selection step was included, both approaches achieved 100\% correct classification.

\subsection{Axonal Organelle Transport}\label{sect.aot}
Deficiencies in the transport of organelles along the neuronal axon have been shown to play an early and possibly causative role in neurodegenerative diseases including Huntington's disease \cite{Gau04}. In \cite{Wi12}, we analyzed time lapse image sequences showing the transport of fluorescently labeled Brain Derived Neurotrophic Factor (BDNF) organelles in a wild-type (healthy) population of mice as well as in a mutant huntingtin protein population. The goal of this study was to examine the relationship between BDNF transport and Huntington's disease.  The transport of the fluorescently labeled BDNF organelles was analyzed using a newly developed multi-target tracking approach we termed "Multitemporal Association Tracking" (MAT). In each image sequence, organelles were segmented and then tracked using MAT and instantaneous velocities were calculated for all tracks. 

Image data was collected over eight time-lapse experiments, with each experiment containing two sets of simultaneously captured image sequences, one for the diseased population and one for the wild type population. There were a total of 88 movies from eight data sets. Although the pairwise NCD was not able to accurately differentiate these populations for individual image sequences, by aggregating the image sequences so that all velocity data from a single experiment and population were considered together, we were able to correctly classify six of the eight experiments as wild type versus diseased for 75\% correct classification accuracy. Analyzing the velocity data from the individual image sequences using pairwise NCD with a minimum distance classifier, we were able to classify 57\% [0.47,0.67] of the image sequences correctly into wild type versus diseased populations. Using the NCD for multisets formulation described in \eqref{eq.class} with the same minimum distance approach, as described in the previous sections, we achieved a classification accuracy of 97\% [0.91,0.99].

\subsection{NIST handwritten digits}\label{sect.NIST}

In addition to the previous applications, we applied the new NCD  for 
multisets to analyzing handwritten digits 
from the MNIST handwritten digits database \cite{Lec98}, 
 a free and publicly available version of the  NIST handwritten digits database 19 that was analyzed in \cite{CV04}.
 The NIST data consists of 128x128 binary images while 
the MNIST data has been normalized
to a 28x28 grayscale (0,..,255) images. 
The MNIST database contains a total of 70,000  handwritten digits. 
Here, we consider only the first 1000 digits. 
The images are first scaled by a factor of four and then adaptive thresholded using an Otsu transform
to form a binary image. The  images  are next converted to one-dimensional
 streams of binary digits  and used to form a pairwise distance matrix between each of the 1000 digits. Originally the input looks as Figure~\ref{fig.digits}.
\begin{figure}[htb]
\begin{center}
\includegraphics[width=3.5in]{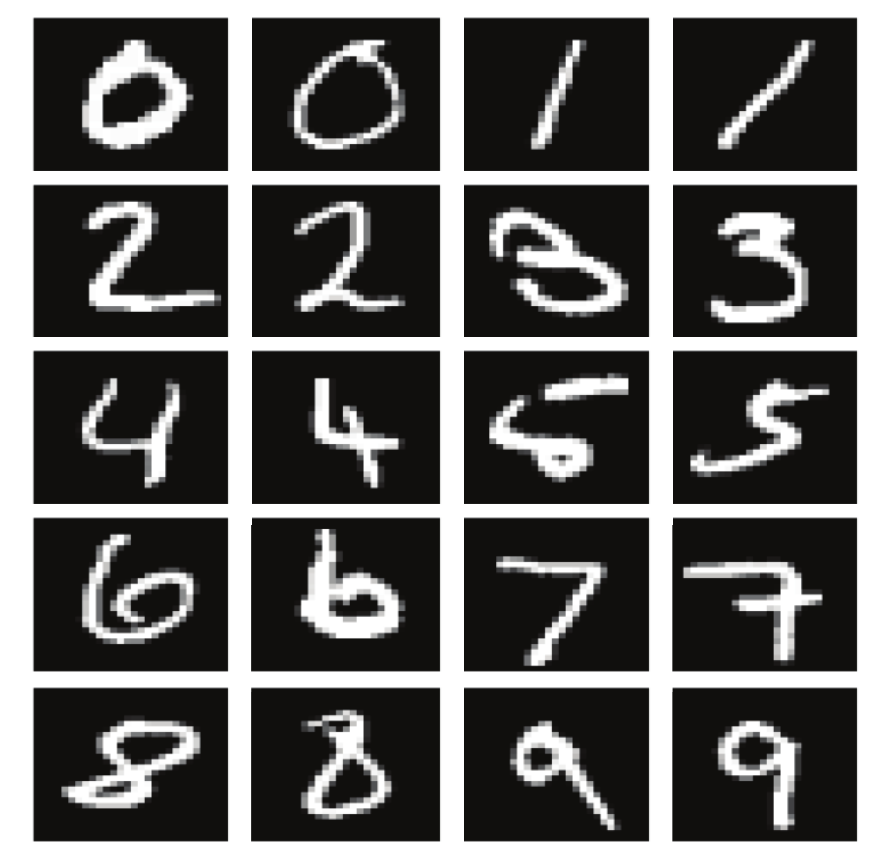}
\end{center}
\caption{Example MNIST digits. Classification accuracy for this application was improved by combining the proposed NCD for multisets with the pairwise NCD.}
\label{fig.digits}
\end{figure}

 Following the same approach as described for the retinal progenitor cells above, 
 we form a spectral matrix from this pairwise distance matrix.  In \cite{CV04}, a novel approach was developed for
 using the distances as input to a support vector machine. Random data examples along with unlabelled images of the same size
were selected and used as training data, achieving a classification accuracy of 87\% on the unscaled NIST database 19 digits. 
 We follow the same approach of incorporating
the distances into a supervised learning framework, using our
 spectral matrix as input to an ensemble of discriminant (Gaussian mixture model) classifiers \cite{Guo07}.
 Using leave-one-out cross validation,  this approach  using the pairwise 
NCD achieved 82\% correct classification [0.79,0.84] for the 1000 scaled and 
resized MNIST digits.

In applying the multisets NCD to this data, we  measured the separation between 
classes or the {\em margin}. Given multisets $A$ and $B$, each corresponding to 
a class in the testing data, we measure the separation between the two classes 
as 
\begin{equation}\label{eq.2}
NCD_1(AB)-NCD_1(A)-NCD_1(B).
\end{equation}
This follows directly from the relevant Venn diagram. Our goal is to 
partition the input classes such that the separation between classes is larger than any separation 
between subsets of the same class, subject to a minimum class size. We have found that this approach works well in practice. We have developed an  
expectation maximization algorithm to partition the classes such that there exist 
no subsets of a class separated by a margin larger than the minimum separation 
between classes. 

Our expectation maximization algorithm attempts to partition the classes 
into maximally separated subsets as measured by \eqref{eq.2}. This algorithm, 
that we have termed {\em K-Lists}, is modeled after the K-means algorithm. 
Although it is suitable for general clustering, here we use it to partition 
the data into two maximally separated subsets. The algorithm is detailed 
in Figure~\ref{tab.1}. There is one important difference between proposed 
K-Lists algorithm and the K-Means algorithm. Because we are not using the centroid
of a cluster as a representative value as in K-Means, but rather the subset 
itself via the NCD for multisets, we only allow a single element to change 
subsets at every iteration. This prevents thrashing where groups of elements 
chase each other back and forth between the two subsets. the algorithm is run until it 
either can not find  any partitions in the data that are separated by more than the
 maximal inter-class separation, or until it encounters a  specified minimum
  cluster size. This step is 
computationally demanding, but it is an inherently parallel computation.

For the retinal progenitor cell data  and synthetic data sets described in the previous sections, 
the K-Lists partitioning algorithm was not able to find any subsets that had a larger separation 
as measured by \eqref{eq.2} compared to the separation between the classes. For the NIST handwritten digits data, 
the partitioning algorithm was consistently able to find subsets with 
separation larger than the between class separation.  The partitioning was run
for a range of different minimum cluster sizes 
(10\%, 20\% and 30\% of the original class 
size). This results in multiple distances to each original digit class. Here we included the two minimum distances to each class as input to the ensemble of discriminant classifiers. This resulted in a classification accuracy of for the 30\% partition size of 85\% [0.83,0.87]. The other two partition sizes had marginally lower classification accuracy.  Finally, we combined the two minimal class distances from the partitioned multisets data along with the pairwise  spectral distances described above as input to the classification algorithm, 
resulting in a combined leave-one-out cross validation 
accuracy of 99.6\% correct [0.990,0.998], a significant improvement over the accuracy achieved using either the pairwise or multisets NCD alone. The current state of the art classifier for the MNIST data achieves an accuracy of 99.77\% correct \cite{Schmid}.

\subsection{Data, Software, Machines}\label{sect.means}
All of the software and the time sequence data for the RPC fate outcome problem 
can be downloaded from http://bioimage.coe.drexel.edu/NCDM. 
The software is implemented in C and uses MPI for parallelization. All data compression 
was done with \texttt{bzip2} using the default settings. Data import is handled by a 
MATLAB script that is also provided. The software 
has been run on a small cluster, consisting of 100 (hyperthreaded) 
Xeon and i7 cores running at 2.9 Ghz. The RPC and synthetic classification runs in approximately
20 minutes for each question, while the partitioning and classification of the 
NIST digit data takes multiple hours for each step. 

\begin{figure}
\begin{enumerate}
\item (Initialize) Pick two elements (seeds) of $X$ at random, assigning one 
element to each $A$ and $B$. For each remaining element $x$, assign $x$ to the 
closer one of $A$ or $B$ using pairwise NCD to the random seeds
\item
For each element $x$, compute the distance from $x$ to class $A$ and $B$
using \eqref{eq.class} and assign to whichever class achieves the smaller distance.
\item
Choose the single element that wants to change subsets, e.g. from $A$ to $B$ 
or vice versa and whose change maximizes $NCD_1(AB)-NCD_1(A)-NCD_1(B)$ and swap 
that element from $A$ to $B$ or vice versa. 
\item
Repeat steps 2 and 3 until no more elements want to change subsets or until 
we exceed e.g. 100 iterations.
\end{enumerate}
Repeat the whole process some fixed number of times (here we use 5) for each $X$ 
and choose the subsets that achieve the maximum of $NCD_1(AB)-NCD_1(A)-NCD_1(B)$. If that 
value exceeds the minimum inter-class separation and the subsets are not smaller 
than the specified minimum size then divide $X$ into $A$ and $B$
and repeat the process for $A$ and $B$. If the value does not exceed the minimum 
inter-class separation of our training data or the subsets exceed the specified minimum 
size, then accept $X$ as approximately 
monotonic and go on to the next class.
\caption{Partitioning algorithm for identifying maximally separated subsets
For each class (multiset) $X$, partition $X$ into two 
subsets $A$ and $B$ such that $NCD_1(AB)-NCD_1(A)-NCD_1(B)$ is a maximum}
\label{tab.1}
\end{figure}

\section{Conclusions and Discussion}\label{sect.concl}
Information distance \cite{BGLVZ98} uses the notion of Kolmogorov complexity 
to identify the metric where the pairwise distance is 
a quantification of the dominant of all differences 
between the digital objects. 
A \textit{normalized} form of information distance, 
the similarity metric, was developed in \cite{Li03} allowing the differences
to be relative rather than absolute. With
a leap of faith the Kolmogorov complexities involved were approximated
with real-world compressors. In \cite{CV04} this led to the NCD, with the
real compressors involved required to satisfy certain properties
and under this requirement the NCD is shown to be a metric as well.
Moreover the evolutionary trees were liberated to hierarchical 
clustering of general objects.
The NCD uses standard file compression algorithms (\textit{e.g.} 
\texttt{gzip}, \texttt{bzip2} or \texttt{ppmz}). In \cite{Li08}, it was 
proposed to extend the idea of information distance to multiple objects, 
or multisets. The idea of information distance for multisets was studied 
in \cite{Vi11}, where certain properties were shown, 
including the positive definiteness, 
symmetry and triangle inequality properties required of a metric. 
In order to compare the relative rather than absolute differences among 
objects, a normalized form of this information distance is required. 
Previous efforts to find such a normalized form that obey 
the triangle inequality were not successful \cite{Vi11}. 

Here we present a normalized form of the information distance for 
multisets and show it to be a metric. This normalized form 
of information distance is based on the observation that in order 
to obey the triangle inequality, the distance can never decrease 
as elements are added to the multiset. The proposed normalized 
information distance for multisets reduces to the original 
(pairwise) formulation of normalized information distance as in 
\cite{Li03} when applied to multisets of cardinality two. 
Similar to the original pairwise NCD developed as an approximation to 
the normalized information distance \cite{CV04}, 
we developed a practically effective approximation to the normalized 
information distance for multisets based on file compression algorithms. 
This NCD for multisets is more computationally demanding than the pairwise 
NCD but it is straightforward to implement the computations in parallel. 

The NCD for multisets is applied to previous applications where 
the pairwise NCD was used so that comparison is
possible. In all these applications it
improved the results either alone or in combination with the pairwise version.
In some applications, including retinal progenitor cell fate prediction 
and the analysis of simulated populations of proliferating cells, 
the new NCD for multisets obtain significant improvement over the pairwise NCD.
In other applications such as the NIST handwritten digits, 
the NCD for multisets alone did not significantly improve upon the result 
from the pairwise NCD, but a significant overall improvement in 
accuracy resulted
by combining both distance measures. In all cases, we applied 
the same parameter-free implementation of both the multiple version
and the pairwise version of the NCD. That is, no features of the 
problems were used at all.

\bibliographystyle{plain}

\begin{thebibliography}{99}

\bibitem{AS05}
C. An\'e and M. Sanderson,
Missing the forest for the trees: phylogenetic compression and
its implications for inferring complex evolutionary histories,
{\em Systematic Biology}, 54:1(2005), 146--157.



\bibitem{BGLVZ98}
C.H.~Bennett, P.~G\'acs, M. Li, P.M.B.~Vit\'anyi, and W.~Zurek.
Information distance, {\em IEEE Trans. Inform. Theory},
44:4(1998), 1407--1423.



\bibitem{Ca03}
M. Cayouette, B. A. Barres, and M. Raff, Importance of intrinsic mechanisms in 
cell fate decisions in the developing rat retina, {\em Neuron}, 40(2003), 897--904.

\bibitem{CV04}
R.L. Cilibrasi, P.M.B. Vit\'anyi, Clustering by compression,
{\em IEEE Trans. Inform. Theory}, 51:4(2005), 1523- 1545.

\bibitem{CV07}
R.L. Cilibrasi, P.M.B. Vit\'anyi, The Google similarity distance, 
{\em IEEE Trans. Knowledge and Data Engineering}, 19:3(2007), 370-383.

\bibitem{Co09}
A. R. Cohen, C. Bjornsson, S. Temple, G. Banker, and B. Roysam, 
Automatic summarization of changes in biological image sequences using 
Algorithmic Information Theory, {\em IEEE Trans. Pattern Anal. Mach. Intell.}
31(2009), 1386--1403.

\bibitem{Co10}
A. R. Cohen, F. Gomes, B. Roysam, and M. Cayouette, "Computational prediction 
of neural progenitor cell fates, {\em Nature Methods}, 7(2010), 213--218.

\bibitem{Gau04}
L.R. Gauthier, et al., Huntingtin controls neurotrophic support 
and survival of neurons by enhancing BDNF vesicular transport along 
microtubules. {\em Cell}, 118:1(2004), 127--138.

\bibitem{Guo07}
Guo, Y., T. Hastie, and R. Tibshirani. Regularized Discriminant Analysis and Its Application in Microarray. 
Biostatistics, Vol. 8, No. 1, pp. 86???100, 2007. 

\bibitem{Ka03}
S. D. Kamvar, D. Klein, and C. D. Manning, Spectral learning, 
Proc. Int. Joint Conf. Artificial Intelligence, 2003, 561--566.

\bibitem{Ke04}
E. Keogh, S. Lonardi, and C. A. Ratanamahatana, Towards parameter-free data 
mining, Proc. 10th ACM SIGKDD Int. Conf. Knowledge Discovery and Data Mining, 2004,206--215.

\bibitem{KJ04}
S.R. Kirk and S. Jenkins,
Information theory-based software metrics and obfuscation,
{\em Journal of Systems and Software}, 72(2004), 179-186.

\bibitem{KKKP06}
A. Kocsor, A. Kert\'esz-Farkas,
L. Kaj\'an, and S. Pongor,
Application of compression-based distance
measures to protein sequence classification:
a methodology study,
{\em Bioinformatics}, 22:4(2006), 407--412.


\bibitem{Ko65}
A.N. Kolmogorov,
{Three approaches to the quantitative definition of information},
{\em Problems Inform. Transmission} 1:1(1965), 1--7.

\bibitem{Kr49}
L.G. Kraft, A device for quantizing, grouping, and coding amplitude 
modulated pulses, MS Thesis, EE Dept., Massachusetts Institute of Technology,
Cambridge. Mass., USA.

\bibitem{Lec98}
Y. LeCun, L. Bottou, Y. Bengio, and P. Haffner. "Gradient-based learning applied to document recognition." 
Proceedings of the IEEE, 86(11):2278-2324, November 1998.

\bibitem{Le74}
L.A. Levin,
Laws of information conservation (nongrowth) 
and aspects of the foundation of probability theory,
{\em Probl. Inform. Transm.}, 10(1974), 206--210.

\bibitem{Li03}
M. Li, X. Chen, X.~Li, B.~Ma, P.M.B.~Vit\'anyi.
The similarity metric, {\em IEEE Trans. Inform. Theory}, 50:12(2004),
3250- 3264.

\bibitem{Li08}
M. Li, C. Long, B. Ma, X. Zhu, Information shared by many objects,
Proc. 17th ACM Conf. Information and Knowledge Management,
2008, 1213--1220.

\bibitem{Li11}
M. Li, Information distance and its extensions, 
Proc. Discovery Science, 
Lecture Notes in Computer Science, Volume 6926, 2011, 18--28

\bibitem{LV08}
M.~Li and P.M.B.~Vit\'anyi.
{\em An Introduction to Kolmogorov Complexity
and its Applications}, Springer-Verlag, New York, Third edition, 2008.

\bibitem{Mc56}
B. McMillan, Two inequalities implied by unique decipherability, 
{\em IEEE Trans. Information Theory}, 2:4(1956), 115???-116.


\bibitem{Mu02}
An.A. Muchnik, Conditional complexity and codes,
{\em Theor. Comput. Sci.}, 271(2002), 97--109.

\bibitem{NPAetal08}
M. Nykter, N.D. Price, M. Aldana, S.A. Ramsey,  S.A. Kauffman, L.E. Hood,
O. Yli-Harja, and I. Shmulevich,
Gene expression dynamics in the macrophage exhibit criticality,
{\em Proc. Nat. Acad. Sci. USA}, 105:6(2008), 1897--1900.

\bibitem{NPLetal08}
M. Nykter, N.D. Price, A. Larjo, T. Aho,  S.A. Kauffman, O. Yli-Harja
 and I. Shmulevich,
Critical networks exhibit maximal information diversity in
structure-dynamics relationships,
{\em Physical Review Lett.}, 100(2008), 058702(4).

\bibitem{Schmid}
J. Schmidhuber, Multi-column deep neural networks for image classification. 
Proc. IEEE Conference Comput. Vision Pattern Recognition, 2012, 3642--3649.

\bibitem{Vi11}
P.M.B. Vit\'anyi, Information distance in multiples, 
{\em IEEE Trans. Inform. Theory}, 57:4(2011), 2451-2456. 

\bibitem{Wi11}
M. Winter, E. Wait, B. Roysam, S. Goderie, E. Kokovay, S. Temple, et al., 
Vertebrate neural stem cell segmentation, tracking and lineaging with validation 
and editing, {\em Nature Protocols}, 6(2011), 1942--1952.

\bibitem{Wi12}
M. R. Winter, C. Fang, G. Banker, B. Roysam, and A. R. Cohen, Axonal transport 
analysis using Multitemporal Association Tracking, {\em Int. J. Comput. Biol.
Drug Des.}, 5(2012), 35--48.

\bibitem{Witten2005}
Witten, I. H. and E. Frank (2005). Data Mining: Practical Machine Learning Tools and Techniques.

\bibitem{WLB08}
W. Wong, W. Liu, M. Bennamoun, Featureless Data Clustering, pp 141--164
(Chapter IX) in:
{\em Handbook of Research on Text and Web Mining Technologies},
Idea Group Inc., 2008.


\bibitem{ZHZL07}
X. Zhang, Y. Hao, X. Zhu, M Li,
Information distance from a question to an answer,
Proc. 13th ACM SIGKDD Int. Conf. Knowledge Discovery and Data Mining,
2007, 874--883.

\bibitem{ZL70}
A.K. Zvonkin and L.A. Levin,
The complexity of finite objects and the development of the concepts
  of information and randomness by means of the theory of algorithms,
{\em Russian Math. Surveys} 25:6 (1970) 83-124.


\end{thebibliography}

\end{document}